\documentclass[letterpaper, 10pt, journal, twoside]{IEEEtran}

\usepackage[utf8]{inputenc}
\usepackage{amsmath}
\usepackage{amssymb}
\usepackage[normalem]{ulem}

\usepackage{amsthm}
\usepackage{graphicx}
\usepackage{support-caption}
\usepackage{caption}
\usepackage{subcaption}
\usepackage{cite}
\usepackage{authblk}

\usepackage{cancel}
\usepackage{todonotes}

\newtheorem{thm}{Theorem}[section]

\newtheorem{cor}{Corollary}

\theoremstyle{definition}

\newcommand{\bsrevA}[1]{{\color{black}{#1}}}
\newcommand{\bsrevB}[1]{{\color{black}{#1}}}
\newcommand{\bsrevC}[1]{{\color{black}{#1}}}

\newcommand{\xxnote}[3]{}
\ifx\hidenotes\undefined
  \usepackage{color}
  \renewcommand{\xxnote}[3]{\color{#2}{#1: #3}}
\fi

\usepackage{algorithm}
\usepackage[noend]{algpseudocode}

\newcommand{\cspace}{\mathcal{C}}
\newcommand{\free}{{free}}
\newcommand{\obs}{{obs}}
\newcommand{\path}{\xi}
\newcommand{\unpath}{\hat{\path}} 
\newcommand{\pathsg}{\path_{s \rightarrow g}}

\newcommand{\reals}{\mathbb{R}}

\newcommand{\configuration}{q}
\newcommand{\config}{\configuration}

\newcommand{\union}{\cup}
\newcommand{\clearance}{\delta}
\newcommand{\inflation}{\epsilon}

\newcommand{\graph}{\mathcal{G}}
\newcommand{\ungraph}{\mathcal{\hat{G}}} 
\newcommand{\vertices}{\mathcal{V}}
\newcommand{\edges}{\mathcal{E}}
\newcommand{\vertex}{v}
\newcommand{\node}{\vertex}

\newcommand{\edge}{e}

\newcommand{\layer}{L}
\newcommand{\depth}{d}
\newcommand{\depthmax}{D}

\newcommand{\executioncost}{c_x}
\newcommand{\ec}{\executioncost}
\newcommand{\planningcost}{c_p}
\newcommand{\pc}{\planningcost}

\newcommand{\admissibleHeuristic}{\hat{h}_x}
\newcommand{\ah}{\admissibleHeuristic}
\newcommand{\SDHeuristic}{\hat{h}_{SD}}
\newcommand{\sdh}{\SDHeuristic}
\newcommand{\timeHeuristic}{\hat{h}_p}
\newcommand{\timeh}{\timeHeuristic}

\newcommand{\ftrue}{f^*}
\newcommand{\fest}{\hat{f}}
\newcommand{\gest}{\hat{g}}
\newcommand{\gtrue}{g^*}
\newcommand{\htrue}{h^*}
\newcommand{\successor}{\node^\prime}
\newcommand{\pathLi}{\path_i} 

\newcommand{\sdw}{w_t}

\newcommand{\start}{s}
\newcommand{\goal}{g}

\newcommand{\badnode}{\node_{z}}
\newcommand{\goodnode}{\node_{m}}
\newcommand{\goodnodepredindex}{m-1}
\newcommand{\goodnodepred}{\node_{\goodnodepredindex}}

\newcommand{\astar}{\textsc{A\textsuperscript{*}}}
\newcommand{\astarns}{\textsc{A\textsuperscript{*}}}

\newcommand{\LazySP}[0]{\textsc{LazySP }}

\newcommand{\SPSB}[2]{\rlap{\textsuperscript{#1}}\textsubscript{#2}}
\newcommand{\graphsearch}{\textsc{A\textsuperscript{*} }}
\newcommand{\astareps}{\textsc{A\SPSB{*}{$\epsilon$}}}

\newcommand{\eref}[1]{(\ref{#1})}

\newcommand{\figref}[1]{Fig.~\ref{#1}}
\newcommand{\algoref}[1]{Algorithm~\ref{#1}}
\newcommand{\algolineref}[1]{Line~\ref{#1}}
\newcommand{\thmref}[1]{Theorem~\ref{#1}}

\newcommand{\tabref}[1]{Table~\ref{#1}}

\author{Brad Saund$^1$ and Dmitry Berenson$^1$
  \thanks{Manuscript received: September, 9, 2019; Revised December, 4, 2019;
    Accepted January, 19, 2020.}
  \thanks{This paper was recommended for publication by Editor Nancy Amato upon
    evaluation of the Associate Editor and Reviewers' comments.
    This work was supported in part by NSF under grant IIS-1750489 and by Toyota Research Institute (TRI). This article solely reflects the opinions of its authors and not TRI or any other Toyota entity.}
  \thanks{$^{1}$Authors are with the Robotics Department, University of Michigan, Ann Arbor, MI, USA.
    {\tt\small \{bsaund, dmitryb\}@umich.edu}}
  \thanks{© 2020 IEEE.  Personal use of this material is permitted.  Permission from IEEE must be obtained for all other uses, in any current or future media, including reprinting/republishing this material for advertising or promotional purposes, creating new collective works, for resale or redistribution to servers or lists, or reuse of any copyrighted component of this work in other works.}

}

\title{Fast Planning Over Roadmaps via \\Selective Densification}

\date{November 2019}

\begin{document}

\maketitle

\begin{abstract}
  We propose the Selective Densification method for fast motion planning through configuration space. We create a sequence of roadmaps by iteratively adding configurations. We organize these roadmaps into layers and add edges between identical configurations between layers. We find a path using best-first search, guided by our proposed estimate of remaining planning time. This estimate prefers to expand nodes closer to the goal and nodes on sparser layers.

  We present proofs of the path quality and maximum depth of nodes expanded using our proposed graph and heuristic. We also present experiments comparing Selective Densification to bidirectional RRT-connect, as well as many graph search approaches. In difficult environments that require exploration on the dense layers we find Selective Densification finds solutions faster than all other approaches.
\end{abstract}



\section{Introduction}
\label{sec:introduction}

\IEEEPARstart{W}{e} examine the motion planning problem, the task of finding a valid path
through continuous configuration space from a start to a goal.
\bsrevC{Probabilistically} complete \cite{RRTConnect} and \bsrevC{asymptotically} optimal \cite{RRTstar} algorithms are known, but to be practical algorithms must be fast.
The challenge for a search algorithm is to explore regions that will likely lead to a path, while not performing excessive checking of configurations that do not lead to the goal.


A common planning approach is to precompute a \bsrevA{(probabilistic) roadmap (PRM), which is a} graph where vertices represent configurations and edges represent motions.
Planning can then be reduced to connecting the query start and goal to the graph and finding a path through the graph.
In many robotics problems the obstacles are not known a priori, thus the validity of edges must be checked online.
The computation time of this approach is dominated by node expansions of a graph-search algorithm, and collision checks for edges.
\bsrevA{
  Lazy edge evaluation has proved effective in robotics, with algorithms like LazyPRM \cite{LazyPRM} and the generalization \bsrevA{Lazy Shorted Path} (LazySP) \cite{LazySP} that minimize the number of collision checks.
}

A challenge when precomputing a roadmap is choosing the appropriate
\bsrevA{density of nodes in configuration space.}
\bsrevA{Too few nodes and edges can result in a roadmap with no solutions.}
Increasing the number of nodes in the roadmap increases the computation cost by increasing both the number of node expansions and number of collision-checks during a search.
For many problems dense roadmaps are needed only in some regions while sparse roadmaps perform better in other regions with more free space.
We desire a search method that is able to shift between graphs of different densities, achieving fast performance by searching sparser graphs in most regions and only densifying when required.



Our key contribution is the Selective Densification method for guiding the \bsrevA{density of nodes and edges explored using a heuristic of remaining planning time.}
Specifically, this paper contributes:
\begin{itemize}
\item The Layered Graph: A sequence of roadmaps of increasing density combined into a single graph (\figref{fig:2dLayeredGraph})
\item The Selective Densification Heuristic: An estimate of both remaining execution cost and planning cost used to guide best-first search
\item Bidirectional search in LazySP with search direction during an iteration chosen by minimum planning time
\end{itemize}

\begin{figure}
    \centering
    \begin{subfigure}[b]{0.48\linewidth}
        \centering
        \includegraphics[width=\textwidth, trim={3cm 0 4cm 0},clip]{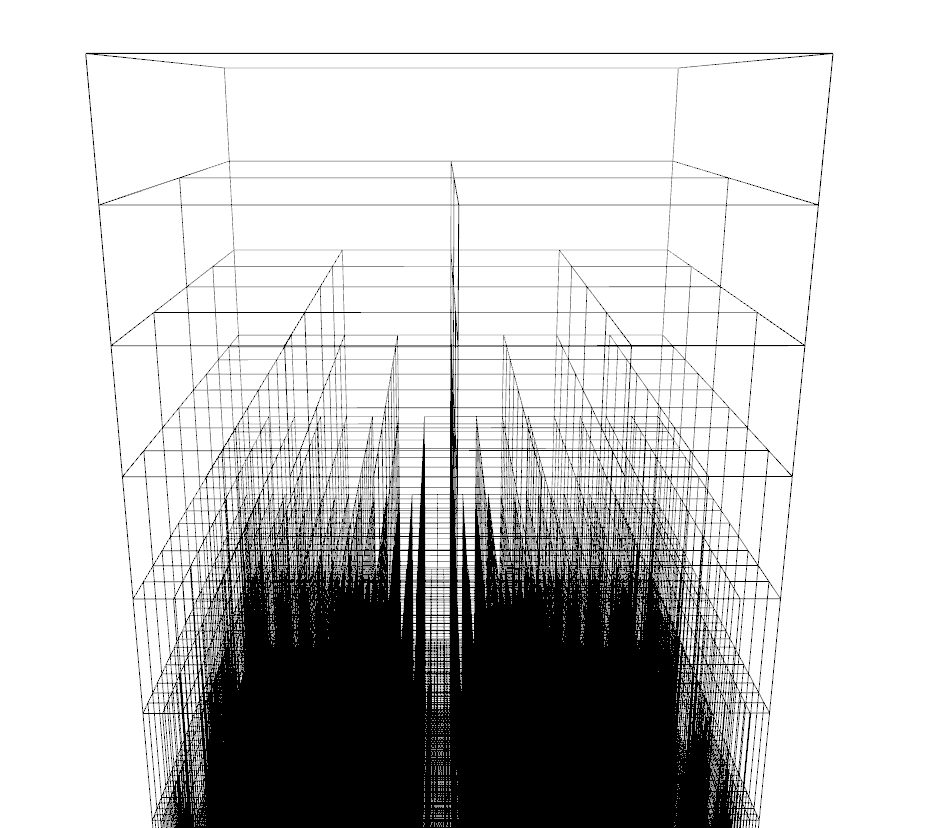}
    \end{subfigure}
    \hfill
    \begin{subfigure}[b]{0.48\linewidth}
        \centering
        \includegraphics[width=\textwidth, trim={4cm 0 3cm 0},clip]{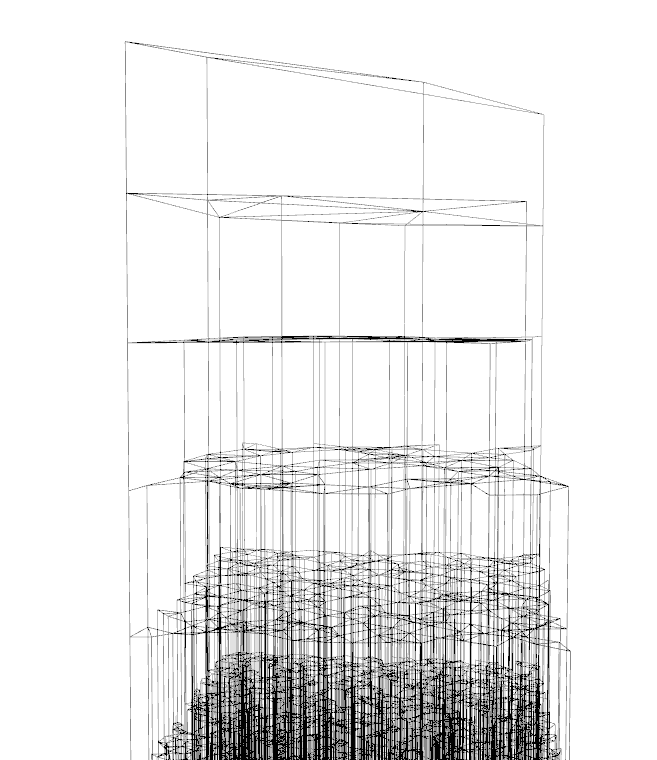}
    \end{subfigure}
    \caption{Layered Graphs in a 2D C-space using a grid structure (left) and halton sequence (right)}
    \label{fig:2dLayeredGraph}
\end{figure}



\bsrevA{We present proofs of the solution quality and maximum depth searched.}
To evaluate our method \bsrevA{in practice}, we perform motion planning experiments on a robotic manipulator arm in simulation and benchmark against existing methods.
\bsrevA{We find that in environments where paths exist on low-density graphs Selective Densification performs similarly to common methods.
In the most challenging environment tested, where paths exist only on high-density graphs, we find Selective Densification performs 4x faster than the next best approach.
}





\section{Related Work}
\label{sec:related_work}

\subsection{Sampling-based C-space planning}
In many robotics domains a graph is constructed to perform search over a continuous space.
Sampling-based motion planning consists of building a graph with an embedding into C-space and then searching the graph for a valid (collision-free) path from a start to a goal.

While approaches such as RRT\cite{RRTConnect} build a tree online, other approaches such as PRM\cite{PRM} can precompute a roadmap and perform online collision checking of edges.
\bsrevA{Both RRT and PRM seek feasible solutions in continuous space without explicit regard to path quality.
Early asymptotically-optimal planners such as RRT$^*$\cite{RRTstar} find minimum cost solutions in the limit as $t\rightarrow \infty$, with improvements such as BIT$^*$\cite{BITstar} seeking faster convergence.
Interestingly, in practice, repeated RRT with shortcut smoothing tends to outperform RRT$^*$ and variants
both in time and path quality \cite{Hauser} \cite{Meijer}.
Given these results, our approach focuses on finding solutions quickly and relies on the shortcut smoothing post-processing to reduce cost.
}

$\astar$\cite{Astar} can be used to search roadmaps, however the edge collision-check is typically the most expensive operation in robotics applications, thus algorithms such as LazyPRM\cite{LazyPRM} and the later generalization LazySP\cite{LazySP} that minimize the number of edge evaluations typically run faster.
\bsrevC{In fact, LazySP is optimal w.r.t. the number of edge checks \cite{Haghtalab2017ThePV} and later work balances time spent on edge checking vs. expansion \cite{Mandalika2018}}.
Using a precomputed roadmap offers the advantage that in static scenes edge validity is constant, so at most one collision check is required across multiple queries.
However even in a single query setting using a precomputed roadmap offers two distinct advantages over a RRT: determinism and the ability to precompute \bsrevC{environment-independent} properties of edges.
\bsrevC{
  We employ both forms by generating a graph in the robot configuration space and precomputing the swept volumes occupied by the robot for some edges.
  This is done without knowledge of the specific environment and prior to query thus we consider this precomputation not time sensitive.

}

\subsection{Densification Strategies}
A challenge when constructing roadmaps is determining the number of vertices and edges needed.
Selecting too few may yield a roadmap without a feasible solution or only a costly solution.
Selecting too many yields a roadmap that is computationally expensive to store and search.
Approaches such as SPARS\cite{SPARS}, bridge sampling \cite{PlanningAlgorithms}, and others therefore build explicit graphs online using specific vertex sampling and edge-connection methods that are aware of obstacles to limit the graph size.
\bsrevC{We pre-compute a much larger graph than these methods generate, but bias search towards the sparser portions.}

Other approaches use multi-resolution roadmaps, connecting a low resolution roadmap to high resolution roadmaps in certain regions specified by \bsrevC{heuristics} \cite{Likhachev2009} \cite{Petereit2013}.
\bsrevC{
  Planning with Adaptive Dimensionality uses a user-specified projection into a lower dimensional space, and reverts to the full dimensional space where the projection is inadequate \cite{Gochev2011} \cite{Cohen2011} \cite{Gochev2013}.
  For example, a projection may map a car to a 2D point or a robotic arm to the end effector location.
  This approach will only construct two graphs of different densities and relies on a user-defined projection that must obey certain properties.
}


Yet another approach is batching, where a fixed roadmap is searched completely before densifying \cite{Choudhury2017AnytimeMP}, \cite{Starek2015AnAS}. This has been used to first find solutions quickly on sparse layers then find better \cite{ChoudhuryDensification}, and asymptotically optimal \cite{Pavone2018} solutions by searching denser layers.
A drawback is the expense of the search of an entire layer before densifying, which can be especially large if no path exists.
As suggested in the future work of LEMUR \cite{Dellin2016} it would be desirable to further densify promising regions before exploring the entire batch.
Our work pursues this approach of selectively densifying certain regions by balancing the expected path length and computational cost at different batch levels.

\subsection{Heuristic Graph Search}
Finding a path in a roadmap requires searching a graph.
While $\astar$ finds optimal paths using an admissible heuristic\cite{Astar}, weighted $\astar$ inflates the heuristic \cite{Pohl1970} achieving bounded sub-optimality (both with \cite{ARAstarAnalysis} and without \cite{OptimisticAstar} using a closed list).
Methods such a ARA* \cite{ARAstar}, ANA* \cite{ANAstar} and numerous others observe that an inflated heuristic and/or a greedy search tend to find solutions faster, trading path cost for planning time, although there is no such guarantee and in some cases this can slow down search \cite{Wilt1977}.
\astareps \cite{Pohl1970}, BUGSY \cite{Ruml2013} and LEMUR \cite{Dellin2016} make this tradeoff explicitly by including various estimates of ``planning-cost-to-go'' in the search heuristic.
We extend this idea to our Layered Graph, on each layer increasing the estimated future planning cost, and therefore inflation factor.






\section{Problem Definition and Notation}
\label{sec:problem_definition}



Let $\cspace$ be a configuration space with free space $\cspace_\free$ and C-space obstacles $\cspace_\obs = \cspace \setminus \cspace_\free$.
$\cspace_\obs$ are represented implicitly via a collision-check function that given a $q \in \cspace$ tests whether $q \in \cspace_\free$. 
A path $\pathsg$ from $q_s$ to $q_g$ is feasible if $\pathsg \subset \cspace_\free$
In practice we relax feasibility by assuming it is sufficient to collision-check a densely discretized $\path_{s \rightarrow g}$.

The execution cost of the straight-line motion between configurations is given by $\ec: \cspace \times \cspace \rightarrow \reals_{\geq 0}$, which for this work we assume is given by a distance: $\ec(q_1, q_2) = ||q_1 - q_2||$.
The execution cost of a feasible discretized path $\pathsg = [q_1, ..., q_n]$ is given by
\begin{align}
  \ec(\pathsg) = \sum_{i=1}^{n-1} \ec(q_{i}, q_{i+1})
\end{align}

Furthermore, a planner incurs a cost $\pc$ in computing a feasible path $\path$.
We consider the case where $\pc$ is the time spent to return $\path$ once $q_s, q_g$ and $\cspace_\obs$ are provided.
\bsrevA{As the primary objective we seek planners that find a feasible path as quickly as possible. }

Define a graph $\graph(\vertices, \edges)$ with vertices $\vertices$ and edges $\edges$.
A roadmap is an embedding of $\graph$ into $\cspace$ such that vertices map to configurations $\configuration \in \cspace$ and edges map to C-space paths $\path_e \subset \cspace$ connecting vertices.
Let an $r$-disk graph be a graph with edges connecting exactly those vertices $v_1, v_2$ for which the corresponding configuration $q_1, q_2$ satisfy $||q_1 - q_2|| < r$.

An edge $e$ is collision-free if and only if the corresponding path $\path_e \subset \cspace_\free$.
As such, a feasible path through a graph from $\node_s$ to $\node_g$ consists of collision-free edges and induces a feasible C-space path $\path$ from $q_s$ to $q_g$.
Each edge has an associated cost, and the cost of a full path is the sum of the edge costs.

For use in lazy edge evaluation, let $\ungraph$ denote the graph $\graph$ initialized with the collision state of all edges labeled as ``unknown''.


\section{Approach: Selective Densification}
\label{sec:approach}

Our approach consists of finding a path on a graph with an embedding into C-space.
The core innovation is the combination of specific graph structure with a heuristically guided search that balances execution cost and expected remaining planning cost.
The graph has connected layers of different densities allowing the search to traverse to denser layers to navigate through tight spaces with precision and then return to sparser layers that are searched more quickly.
The graph is searched using $\astar$ with an inadmissible heuristic estimating both path cost and planning cost.
As the true remaining planning cost is unknown, it is estimated via a heuristic that grows as the layer depth increases.

\subsection{Graph Structure}
Consider a sequence of unique configurations $Q = (q_1, q_2, \dots)$, and a strictly increasing sequence of positive integers $(n_1 < n_2 < ... n_\depthmax)$.
A layer $L_i$ is an $r$-disk graph constructed from the first $n_i$ configurations of $Q$ and connection radius $r_i$.
Denote a vertex of $L_i$ with $v^i_j$ where $j\leq n_i$.
For each pair of adjacent layers $(L_i, L_{i+1})$ define the inter-layer edges as $\edges_{i \leftrightarrow i+1} = \{e(v^i_j \leftrightarrow v^{i+1}_j) \  \forall j \leq n_i \}$.
Note that $v^i_j$ and $v^{i+1}_j$ represent the same configuration $q_j$, thus the inter-layer edges are zero-cost edges that change layers without changing configurations.
The Layered Graph is defined as $\graph = (\union_i L_i) \union (\union_i \edges_{i \leftrightarrow i+1})$.

We define $\node.\config$ as the configuration associated with vertex $\node$ and $\node.\depth$ 
as the layer number $i$ where $\node \in \layer_i$

Figure \ref{fig:2dLayeredGraph} visualizes two Layered Graphs in a 2D C-space, with the vertical dimension representing layer depth.
Vertical edges are therefore the zero-cost edges connecting vertices in adjacent layers representing the same configuration.

Although this graph can be precomputed, a query may consist of a $q_s$ and $q_g$ that are not in $\graph$, thus during a query vertices for $q_s$ and $q_g$ are added to each layer, edges within each $L_i$ are added determined by the connection radius $r_i$, and inter-layer edges are added for the vertices corresponding to $q_s$ and $q_g$.

\begin{algorithm}
\caption{Selective Densification Search} \label{alg:lazysp}
\begin{algorithmic}[1]
\While{True}
\State $\path \leftarrow \graphsearch (\ungraph, v_s, v_g)$ \label{alg:lazysp:bestfirst}
\If{CheckEdges($\path$)}
\State \textbf{return} $\path$ 
\EndIf
\EndWhile
\end{algorithmic}
\end{algorithm}

\begin{algorithm}
  \caption{
    Bidirectional Selective Densification Search} \label{alg:bilazysp}
  \begin{algorithmic}[1]
    \State $t_{forward} \leftarrow 0$
    \State $t_{backward} \leftarrow 0$
    \While{True}
    \If{$t_{forward}\leq t_{backward}$}
    \State $\path \leftarrow \graphsearch(\ungraph, v_s, v_g)$ \label{alg:bilazysp:forward}
    \State $t_{forward} \leftarrow t_{forward} +$ timeOf(Line \ref{alg:bilazysp:forward})
    \Else
    \State $\path \leftarrow \graphsearch(\ungraph, v_g, v_s)$.reverse()
    \label{alg:bilazysp:backward}
    \State $t_{backward} \leftarrow t_{backward} +$ timeOf(Line \ref{alg:bilazysp:backward})
    \EndIf
    \label{alg:bilazysp:bestfirst}
    \If{CheckEdges($\path$)}
    \State \textbf{return} $\path$ 
    \EndIf
    \EndWhile
  \end{algorithmic}
\end{algorithm}

\begin{algorithm}
  \caption{CheckEdges($\path$)} \label{alg:checkedges}
  \begin{algorithmic}[1]
    \For{$\edge$ in $\path$}
    \If{$\edge$ is not valid} \label{alg:checkedges:collisioncheck} \Comment{collision check}
    \State mark $\edge$ as $invalid$
    \State \textbf{return} False
    \EndIf
    \State mark $edge$ as $valid$
    \EndFor
    \State \textbf{return} True
  \end{algorithmic}
\end{algorithm}

\begin{algorithm}
\caption{\graphsearch($\ungraph, \node_s, \node_g$)} 
\label{alg:bestfirst}
\begin{algorithmic}[1]
  \State $open \leftarrow \{\node_s\}, closed \leftarrow \{\}, \gest(\node) \leftarrow \infty,
  \gest(\node_\start)=0$
\While{$open$ is not empty}
\State $\node \leftarrow open$.pop\_lowest\_$f$value()
\State insert $\node$ into $closed$ 
\If{$\node$ is $\node_g$}
\State \textbf{return} ReconstructPath($\node_s, \node_g$)
\EndIf
\For{edge $\edge$ and successor $\successor$ of $\node$}
\If{$\edge$ is $invalid$} skip
\EndIf
\If{$\gest(\successor) > \gest(\node) + \ec(\edge)$}
\State $\gest(\successor) = \gest(\node) + \ec(\edge)$
\If{$\successor \not\in$  $closed$} \label{alg:bestfirst:inclosed}
\State $\fest(\successor) = \gest(\successor) + \sdh(\successor)$
\State insert $\successor$ into $open$
\EndIf
\EndIf
\EndFor
\EndWhile
\State \textbf{return} No Path Exists
\end{algorithmic}
\end{algorithm}

\subsection{Utility Guided Graph Search}

To answer a query $(q_s, q_g)$ first the corresponding graph nodes $(\node_s, \node_g)$ are found (or added), a best-first (\astar) search with lazy edge evaluation is performed over the graph yielding a path through $\graph$, and the induced C-space path is returned.
Lazy edge evaluation attempts to reduce the number of edge collision-checks, but may require many calls to \astarns.
This approach has been called LazyPRM \cite{LazyPRM} or \LazySP \cite{LazySP}.

Algorithms \{\ref{alg:lazysp} or \ref{alg:bilazysp}\} and \ref{alg:checkedges} show the outer loops that perform the lazy edge evaluation.
Line \ref{alg:lazysp:bestfirst} of \algoref{alg:lazysp} performs a best-first search over $\ungraph$, which is graph $\graph$ with the optimistic assumption that unevaluated edges are collision free.
Edges along this optimistic path are collision-checked \bsrevA{in order} by \algoref{alg:checkedges}.
The result of a collision check is stored so that future iterations of \graphsearch may not traverse invalid edges.

\algoref{alg:bilazysp} implements bidirectional \LazySP as bidirectional search tends to find solutions faster than unidirectional variants.
Bidirectional \LazySP alternates a unidirectional $\astar$ search at each iteration reusing the results of collision checks and does not require the algorithmic machinery of bidirectional variants of $\astar$ to maintain guarantees.
A common implementation of bidirectional search alternates search direction at each iterations.
We introduce this variant that balances the total time spent searching in each direction.
If the search time per iteration is independent of direction, our version is equivalent to alternating direction after each iteration.
However in practice we observe the $\astar$ time is not negligible and can have a huge dependence on direction.
In such cases we empirically observe significantly better performance when balancing total time instead of iterations.


Algorithm \ref{alg:bestfirst} is $\astar$ with invalid edges removed, and using the inadmissible heuristic $\sdh$.
The functions $\gtrue(\node), \htrue(\node),$ and $\ftrue(\node)$ for a node are the optimal cost-to-come, cost-to-go, and cost from the start to goal through $\node$ respectively.
$\gest, \sdh,$ and $\fest$ are estimates of these quantities.


\subsection{Heuristics} \label{sec:PlanningHeuristic}
The heuristic $\sdh$ contains both a heuristic estimate of the execution cost-to-go ($\ah$)  and the planning time-to-go $(\timeh)$.
We choose $\ah$ as the euclidean distance to the goal.
$\ah$ is consistent (and therefore admissible) and the cost is only achieved if there is a straight line path from $\node$ to the goal.
$\timeh$ is proportional to distance $||\node.q - q_g||$, the number of nodes in the layer $n_{\node.\depth}$, and a tunable constant $\sdw$.

\begin{align}
  \ah(\node) &= ||\node.q - q_g|| \label{eq:ah}\\
  \timeh(\node) &= \sdw n_{\node.\depth} ||\node.q - q_g||\\
  \sdh(\node) &= \ah(\node) + \timeh(\node)\\
  &= \ah(\node) (1 + \sdw n_{\node.\depth}) \label{eq:heuristic}
\end{align}

We now discuss our choice of $\timeh$, with formal analysis provided in the next section.
First consider that $\timeh$ is proportional to $||\node.q - q_g||$.
Assuming a maximum edge length, doubling the distance to the goal will double the lower bound on the number of node expansions required (approximately, due to the discretization of nodes).
Mathematically, proportionality to $||\node.q - q_g||$ causes $\timeh$ to appear as an inflation of $\ah$, which has empirically shown to reduce planning times \cite{Pohl1970} \cite{Wilt1977}.

\bsrevB{
  Next, we discuss $\timeh(\node) \propto  n_{\node.\depth}$.
  Consider search over a traditional roadmap composed of a single layer.
  We expect search time to increase for a denser roadmap
  If an oracle guides a search to only expand nodes on the optimal path, we might expect search time to be inversely proportional to the maximal edge.
  However, consider the case of misleading cul-de-sacs, as shown in \figref{fig:culdesac}, where a best-first search must expand nodes in the entire volume of the cul-de-sac, thus the total time is proportional to the total number of nodes.
  We intentionally tailor $\timeh$ to these environments.
}

Finally, consider the constant $\sdw$.
Theoretically $\sdw$ can appear as a weighting between a user's preference between planning time and execution time.
Clearly if a user is agnostic about planning time $\sdw$ should be set to 0, as then $\sdh$ reduces to $\ah$, and thus \graphsearch will yield the optimal path on $\ungraph$.
On the other hand if the goal is simply to find a feasible path as quickly as possible, setting $\sdw$ very high yields a greedy search over each layer, which empirically finds solutions quickly \cite{Wilt1977, ANAstar}.
However, the optimal $\sdw$ might also depend on factors such as the expected path length, past planning experience, expected size of cul-de-sacs, and expected percentage of freespace in $\cspace$.
In practice we observe that regardless of the value of $\sdw$, shortcut smoothing tends to yield paths with similar cost.

\figref{fig:2dSearch} illustrates (unidirectional) Selective Densification search applied to a toy example.
Because $\timeh \propto n_{\node.\depth}$, \graphsearch tends to explore nodes in sparser layers first.
The search first explores the sparsest edges in the top layer but neither is valid.
The search progresses to nodes on denser layers closest to the goal, and then pops back up after navigating the narrow region.
Although the graph extends deep, these edges are never explored, due to the higher planning heuristic of nodes on denser layers.

\bsrevB{As an aside, suppose Selective Densification finds a path well before a robot can begin execution.}
An anytime method that converges to the optimal path could be created through the following modifications.
  Store the cost of the best path found so far $\ec(\path^{Best})$ after \algoref{alg:bilazysp}.
  Add the check $\gest(\successor) + \ah(\successor) < \ec(\path^{Best})$ to \algoref{alg:bestfirst} \algolineref{alg:bestfirst:inclosed} to prevent expansion of nodes that could not possibly lead to better paths.
  Repeat \algoref{alg:bilazysp} until the user requests a result, or until no nodes are on the open list.
  One could imagine a variety of schemes to decrease $\sdw$, but choosing a particular scheme is beyond the scope of this work.



\begin{figure}
    \centering
    \begin{subfigure}[b]{0.45\linewidth}
        \centering
        \includegraphics[width=\textwidth]{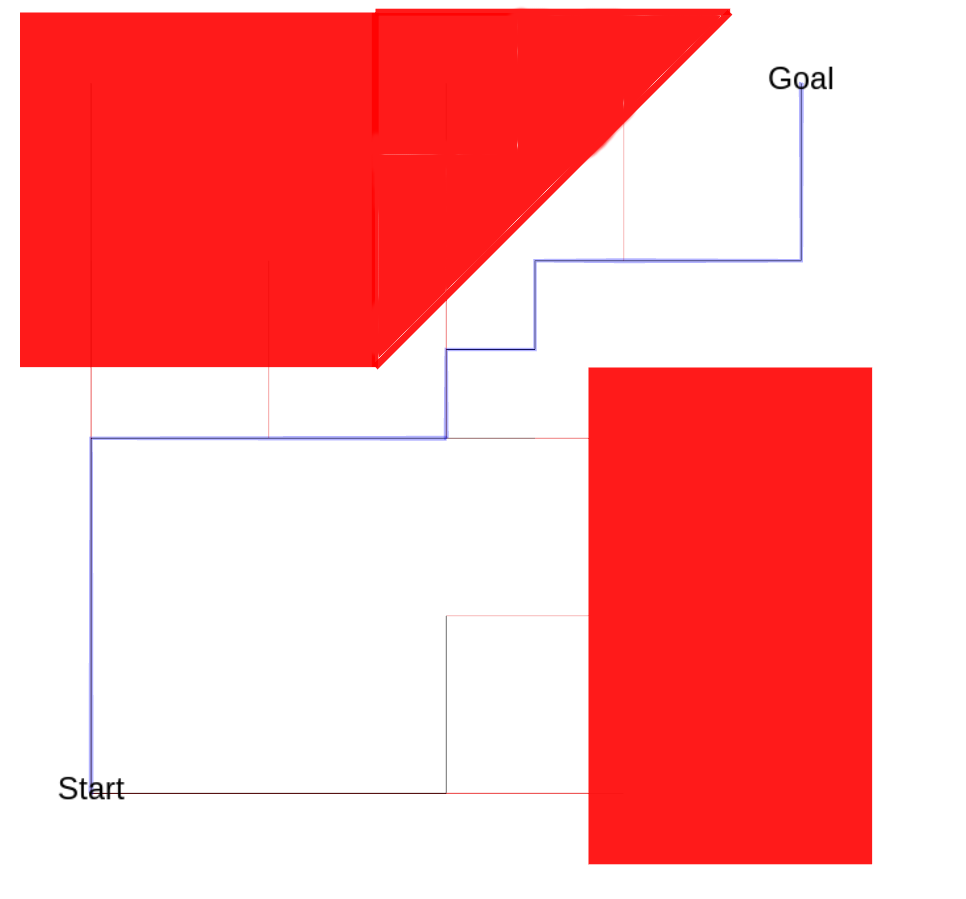}
    \end{subfigure}
    \hfill
    \begin{subfigure}[b]{0.45\linewidth}
        \centering
        \includegraphics[width=\textwidth]{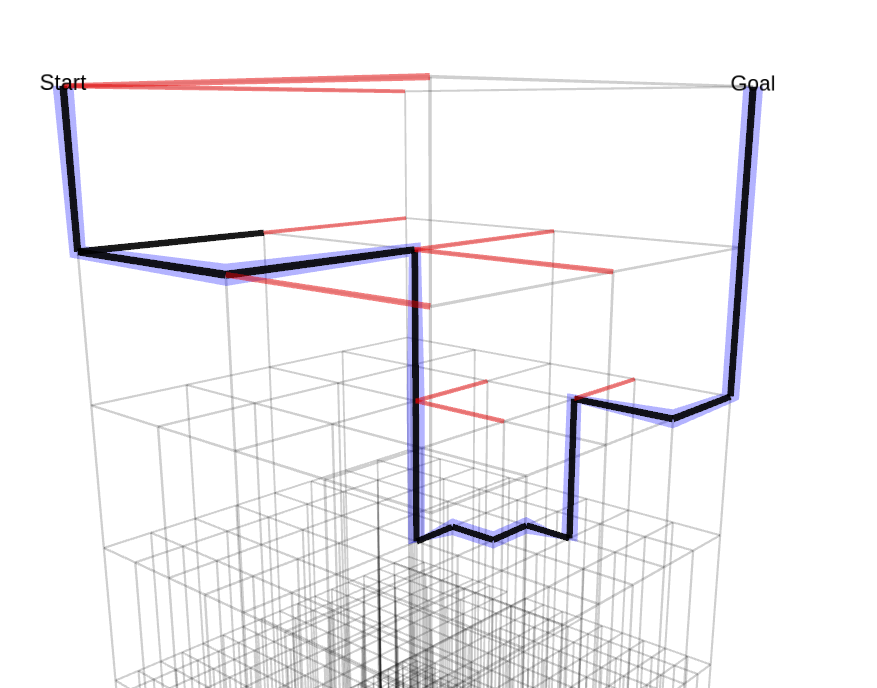}
    \end{subfigure}
    \caption{A 2D search problem solved using Selective Densification with evaluated edges shown in black (valid) and red (invalid) and the final path shown in blue. Left: 2D view with red obstacles. Right: View of Layered Graph with unevaluated edges shown in light grey
    }

    \label{fig:2dSearch}
\end{figure}

\section{Analysis}
\label{sec:analysis}

\begin{figure}
    \centering
    \begin{subfigure}[b]{0.8\linewidth}
        \centering
        \includegraphics[width=\textwidth]{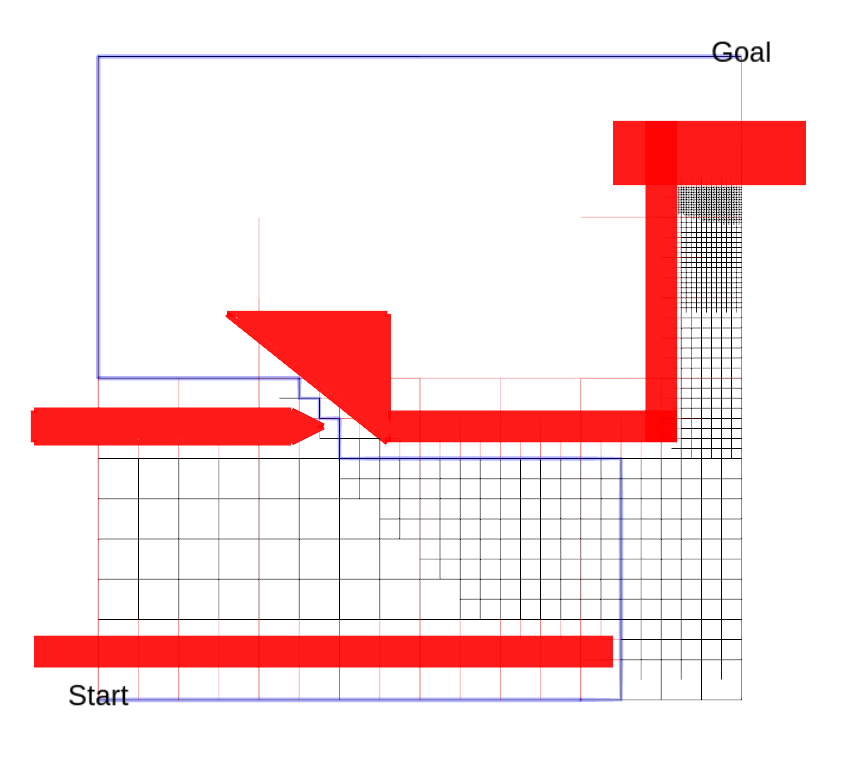}
    \end{subfigure}
    \hfill
    \caption{All edges explored using \algoref{alg:bestfirst} (unidirectional) search with a misleading trap.}
    \label{fig:culdesac}
\end{figure}


In this section we prove bounds on the solution cost and search depth of Selective Densification.
Because $\graph$ contains nodes with different estimated planning times connected by zero-cost (vertical) edges, $\sdh$ is not simply an inflated consistent heuristic, however we build up similar guarantees relying on consistency across each layer of $\graph$.

For this section define $\gtrue_i$ as $\gtrue$ over the single-Layered Graph $\layer_i$.
\bsrevC{
  References to $\astar$ use the proposed heuristic Eq. \eqref{eq:heuristic}, which is written as a per-layer inflation factor $\sdh(\node) = \inflation_i \ah(\node)$ with $\inflation_i = (1 + \sdw n_i)$.
  }

First, consider the \graphsearch search over a single layer ($\graph = \layer_i$). 

\begin{thm}
  \label{thm:boundedg}
  Any $\node$ expanded by $\graphsearch (\layer_i, \node_\start, \node_\goal)$ satisfies $\gest(\node) \leq \inflation_i \gtrue(\node)$.
  
\end{thm}
\begin{proof}
  On a single layer \graphsearch reduces to weighted-$\astar$ with $\sdh$ equivalent to the consistent heuristic $\ah$ inflated by $\inflation_i$, thus we can directly apply the proof presented in \cite{ARAstarAnalysis}, Theorem 10. 
\end{proof}

Next we show this per-layer bound holds on the full graph.

\begin{thm}
  \label{thm:boundedsubopt}
  Consider an optimal path $\pathLi = (\node^i_\start = \node_1, ..., \node_k =  \node^i_\goal)$ on $\layer_i$.
  For every $\node_j \in \pathLi$ expanded by \graphsearch over $\ungraph$ it holds $\gest(\node_j) \leq \inflation_i \gtrue_i(\node_j)$
\end{thm}

\begin{proof}

  \bsrevC{
    Each iteration of $\astar$ expands the node with the lowest $\fest$, thus it is sufficient to show that if a node $\badnode \in \pathLi$ is on the open list with $\gest(\badnode) > \inflation_i \gtrue_i(\badnode)$ then there must be another node on the open list with lower $\fest$-value.
    We show this by induction, by showing that if at some iteration \thmref{thm:boundedsubopt} holds for all $\node \in \pathLi$ expanded so far implies it holds for the next iteration.
    The base case trivially holds, as initially no nodes have been expanded.
  }
  
  \bsrevC{
    Consider some iteration of $\astar$ and assume \thmref{thm:boundedsubopt} holds for all $\node \in \pathLi$ expanded so far.
    Consider some node $\badnode \in \pathLi$ with $\gest(\badnode) > \inflation_i \gtrue_i(\badnode)$, thus with $\fest(\badnode) > \inflation_i (\gtrue_i(\badnode) + \ah(\badnode))$.
    Note that if no such $\badnode$ exists, \thmref{thm:boundedsubopt} trivially holds for the next iteration.
    We show that $\badnode$ will not be chosen for expansion.
    
    Case 1: A node before $\badnode$ along $\pathLi$ has been expanded. There then must be a $\goodnodepred \in \pathLi$ before $\badnode$ along $\pathLi$ with successor $\goodnode \in \pathLi$ on the open list.
    By assumption $\gest(\goodnodepred) \leq \inflation_i \gtrue_i(\goodnodepred)$.
    It follows that $\gest(\goodnode) \leq \inflation_i \gtrue_i(\goodnode)$.
    \begin{align}
      \gest(\goodnode)
      &\leq \gest(\goodnodepred) + \ec(\goodnodepred, \goodnode)\\
      &\leq \inflation_i \gtrue_i(\goodnodepred) + \ec(\goodnodepred, \goodnode)\\
      &\leq \inflation_i \gtrue_i(\goodnode)
    \end{align}
    We show $\goodnode$ has a lower $\fest$-value than $\badnode$ so the assumption will also hold in the next iteration.

    For every node along $\pathLi$ it holds that
    \begin{align}
      \gtrue_i(\node_j) + \ah(\node_j) 
      &\leq \gtrue_i(\node_{j}) + \ah(\node_{j+1}) + \ec(\node_j, \node_{j+1})
      \label{eq:subopt:consistency} \\
      &= \gtrue_i(\node_{j+1}) + \ah(\node_{j+1})
      \label{eq:subopt:gval}
    \end{align}
    
    Eq. \eref{eq:subopt:consistency} is obtained using the consistency of $\ah$ over $\layer_i$.
    Eq. \eref{eq:subopt:gval} uses $\gtrue_i(\node_m) + \ec(\node_m, \node_{m+1}) =\gtrue_i(\node_{m+1})$ along the optimal path.

    This is then used to related the $\fest$-values of $\goodnode$ and $\badnode$:
    \begin{align}
      \fest(\goodnode)
      &= \gest(\goodnode) + \inflation_i \ah(\goodnode) \label{eq:subopt:def}\\
      &\leq \inflation_i (\gtrue_i(\goodnode) + \ah(\goodnode))\\
      &\leq \inflation_i (\gtrue_i(\badnode) + \ah(\badnode)) \label{eq:subopt:jump}\\
      &< \fest(\badnode) \label{eq:subopt:final}
    \end{align}
    Thus $\badnode$ will not be expanded in this iteration.
  }

  \bsrevC{

    
    Case 2: No node before $\badnode$ along $\pathLi$ has been expanded yet.
    Thus there must be some $\node^{i'}_s$ with $i' \leq i$ on the open list.
    That is, the open list will contain either the start node, or a node on a denser layer representing the same configuration as the start node.
    It is straightforward to see $\fest(\node^{i'}_s) < \fest(\node^{i}_s) \forall i' \leq i$.
    Applying the same analysis from Case 1: $      \fest(\node^{i'}_s) < \fest(\badnode)$
  }

\end{proof}

\begin{cor}
  Any $\node^i \in \layer_i$ expanded by \graphsearch on $\graph$ satisfies $\gest(\node) \leq \inflation_i \gtrue_i(\node)$.
\end{cor}
\begin{proof}
  \thmref{thm:boundedsubopt} directly proves this by setting $\node^i$ as the end of $\pathLi$.
  If there is no feasible $\pathLi$ then $\gtrue_i(\node) = \infty$ and this trivially holds.
\end{proof}

As this bound on $\gest$ holds for any node, it must hold for the goal node $\node_\goal$.
This bound on $\gest$ therefore also applies to the path returned by \astar.

\begin{cor}
  Consider an optimal path $\path^*_i$ on a $\layer_i$.
  Selective Densification over $\graph$ returns a path $\path_{SD}$ with execution cost $\executioncost(\path_{SD}) \leq \inflation_i \cdot \executioncost(\path^*_i) \forall i$ with $\inflation_i = 1 + \sdw n_i$.
\end{cor}
\begin{proof}
  Selective Densification consists of a LazySP-style search (\algoref{alg:bilazysp}) using \astar.
  \graphsearch is performed over $\ungraph$, the copy of $\graph$ that is optimistic about unevaluated edges.
  Due to this optimism, $\path^*_i$ will always be valid on $\ungraph$.
  
  \graphsearch returns $\unpath$, a path over $\ungraph$ with edges that may not have been collision checked yet.
  On the final iteration Selective Densification validates all edges in $\ungraph$, returning $\path_{SD}$.

  \begin{align}
    \ec(\path_{SD}) = \ec(\unpath)
    = \gest(\node_\goal)
    \leq \inflation_i \gtrue_i(\node_\goal)
    \leq \inflation_i \ec(\path^*_i)
  \end{align}
\end{proof}

This demonstrates that adding additional layers to $\graph$ can only improve the bound of the execution cost of the path returned by $\astar$.
In particular the larger inflation factor of the dense layers does not worsen the bound from the lower inflation factor of a sparse layer.

\begin{cor}
  With $\sdw = 0$ Selective Densification returns the optimal path.
\end{cor}
\begin{proof}
  $\sdw = 0 \implies \inflation = 1 \implies \ec(\path) = \ec(\path^*)$
\end{proof}




Next we show that when a path exists the $\sdh$ heuristic bounds the deepest layer of a search.
We do this by showing that for a sufficiently deep layer $\sdh$ will be larger than any $\fest$ on the open list.
To make this claim, we first define clearance.
A path $\path$ has clearance $\clearance$ if and only if for every point $p$ on $\path$ the configuration-space ball of radius $\clearance$ centered at $p$ contains only valid configurations.
We assume there exists a path with clearance $\clearance$.

\begin{thm}
  \label{thm:depthbound}
  Consider a Layered Graph $\ungraph$ that contains a feasible solution $\pathLi$ with clearance $\clearance$ on $\layer_i$.
The deepest node expanded by \graphsearch is at most on $\layer_j$ for a $j$ with $\inflation_{j} \geq \inflation_i \cdot \ec(\pathLi) / \clearance $.
\end{thm}

\begin{proof}
  From \thmref{thm:boundedsubopt}
  \begin{align}
    \inflation_i \cdot \ec(\pathLi) \geq \inflation_i \ftrue_i(\node) \geq \fest(\node) \forall \node \in \pathLi \label{eq:worstf}
  \end{align}
  Define $\node_k$ as the first node expanded on layer $\layer_{j+1}$.

  Consider two cases:

  If $||\node_k - \node_\goal|| \geq \clearance$ then \
  \begin{align}
    \fest(\node_k) &\geq \sdh(\node_k)\\
    & \geq \inflation_{j+1} ||\node_k - \node_\goal||\\
    & \geq \inflation_{j+1} \clearance\\
    & > \inflation_{j} \clearance \label{eq:deltabound}
  \end{align}
  \bsrevB{
    Consider a layer $j$ sufficiently deeper than layer $i$ such that $\inflation_{j} \geq \inflation_i \cdot \ec(\pathLi) / \clearance $.
  Then Eq. \eref{eq:deltabound} yields $\fest(\node_k) \geq \inflation_i \cdot \ec(\pathLi)$.
  Thus by Eq. \eref{eq:worstf} all nodes in $\pathLi$ would be expanded before any node on layer $\layer_{j+1}$.
  }

  Otherwise consider $||\node_k - \node_g|| < \clearance$.
  \bsrevB{Since $\node_k$ is the first node expanded on layer $\layer_{j+1}$ the predecessor to $\node_k$ therefore must represent the same configuration on layer $\layer_{j}$.
    However, with $\clearance < r_j$ then $\node_\goal$ is on the open list and \bsrevC{will be expanded} before $\node_k$.
    Since $\node_\goal$ has clearance $\clearance$ the edge would be valid.
    Selective Densification would then terminate without ever expanding $\node_k$.
    }
\end{proof}

\thmref{thm:depthbound} provides insight into potential traps for Selective Densification, as although the number of layers searched is bounded, this bound may be large.
In practice this can occur when nodes are expanded in a cul-de-sac close to the goal, causing the heuristic to be a small fraction of $\fest$, and therefore misleading Selective Densification to explore on a dense layer.
In such a scenario $\timeh$ vastly underestimates the remaining planning time of dense layers.
\figref{fig:culdesac} illustrates the behavior in such an environment (using unidirectional search).

\section{Experiments}
\label{sec:experiments}


\begin{figure*}
  \centering
    \begin{subfigure}[b]{0.32\linewidth}
        \centering
        \textbf{Table}
    \end{subfigure}
    \hfill
    \begin{subfigure}[b]{0.32\linewidth}
        \centering
        \textbf{Bookshelf}
    \end{subfigure}
    \hfill
    \begin{subfigure}[b]{0.32\linewidth}
        \centering
        \textbf{Slot}
    \end{subfigure}
    \hfill
    \begin{subfigure}[b]{0.25\linewidth}
        \centering
        \includegraphics[width=\textwidth, trim={0cm 0 0cm 0},clip]{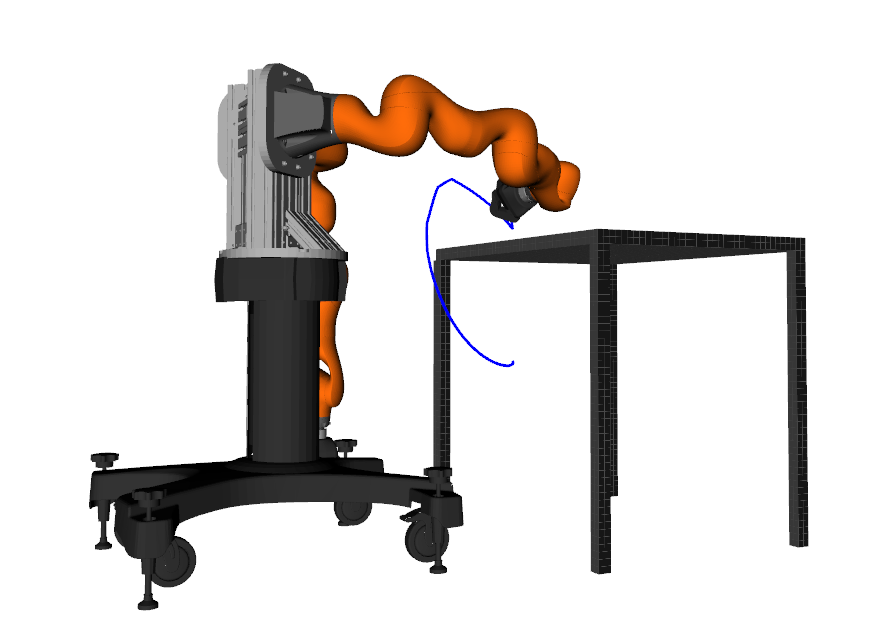}
    \end{subfigure}
    \hfill
    \begin{subfigure}[b]{0.25\linewidth}
        \centering
        \includegraphics[width=\textwidth, trim={0cm 7cm 0cm 0},clip]{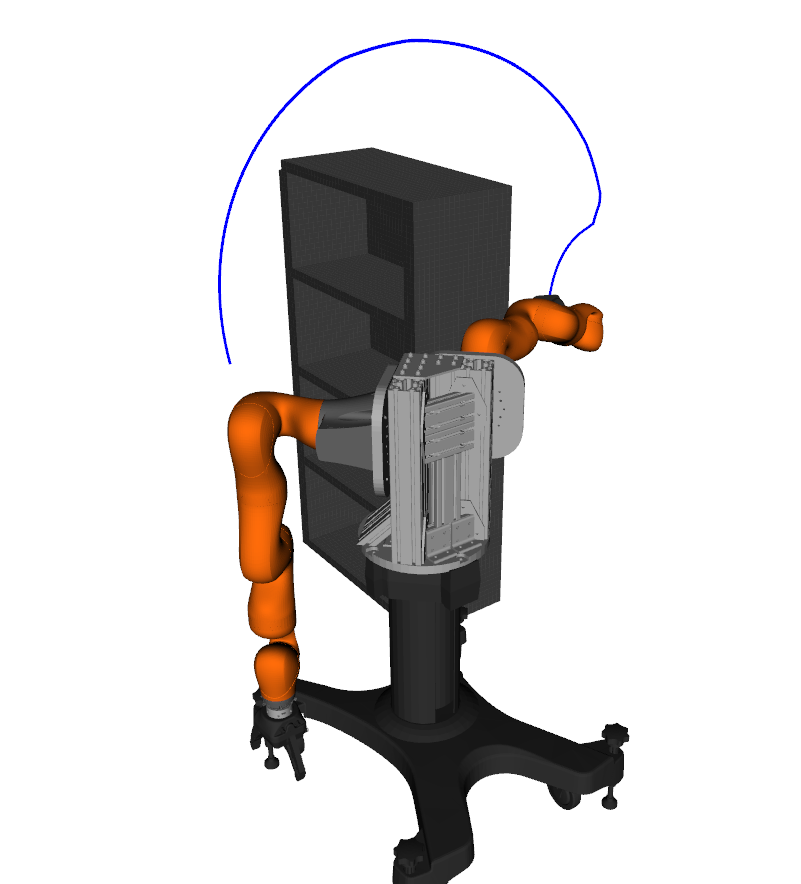}
    \end{subfigure}
    \hfill
    \begin{subfigure}[b]{0.25\linewidth}
        \centering
        \includegraphics[width=\textwidth, trim={0cm 0cm 0cm 0},clip]{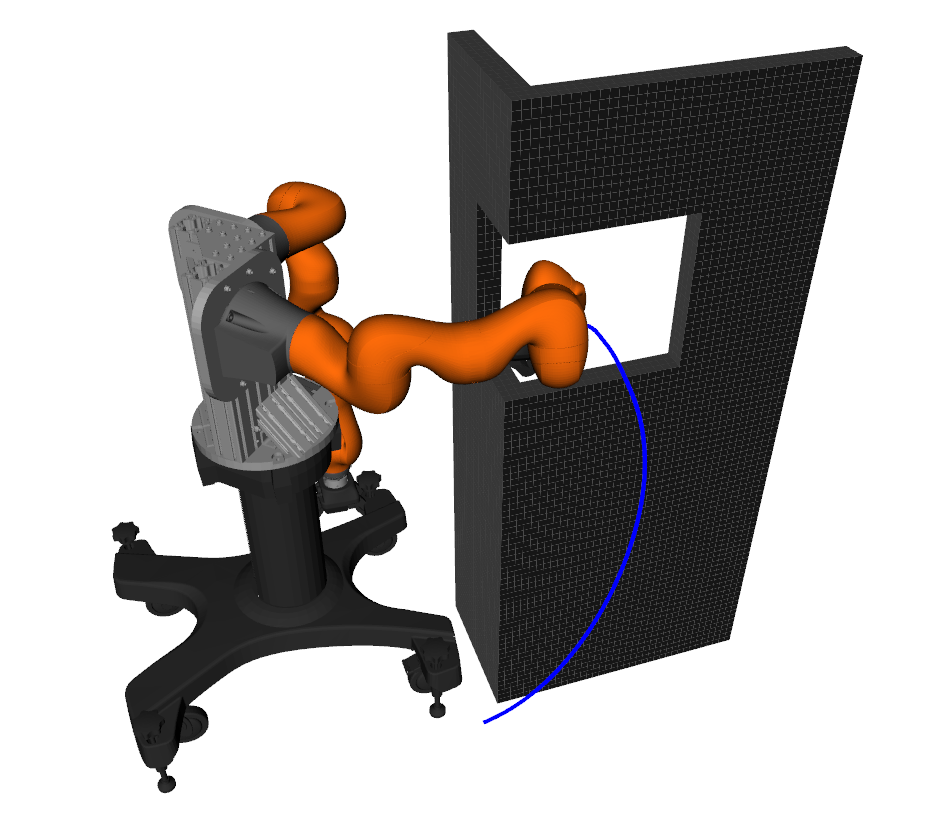}
    \end{subfigure}
    \hfill


    \begin{subfigure}[b]{0.03\linewidth}
      \centering
        \includegraphics[width=\textwidth, trim={0cm 2cm 18.9cm 2cm},clip]{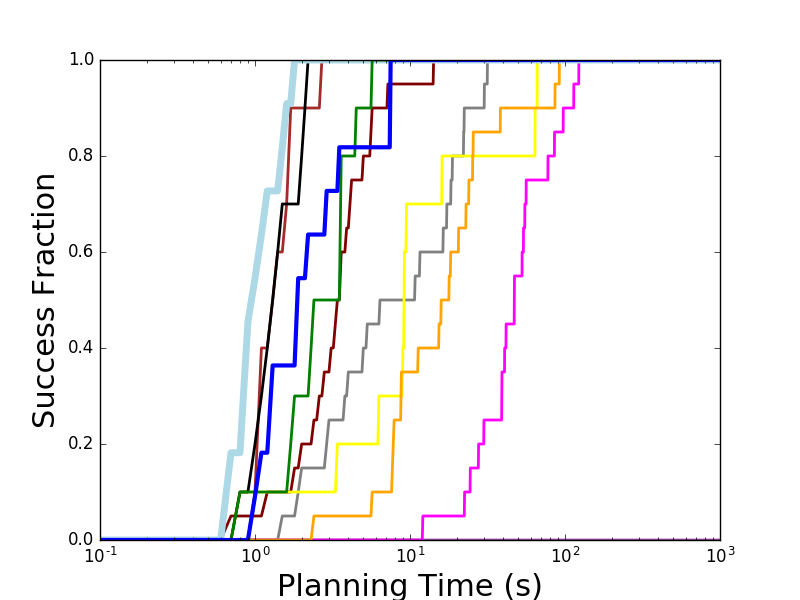}
    \end{subfigure}
    \begin{subfigure}[b]{0.28\linewidth}
        \centering
        \includegraphics[width=\textwidth, trim={1.5cm 0 1.7cm 0cm},
          clip]{img/Table_success_rate.png}
    \end{subfigure}
    \hfill
    \begin{subfigure}[b]{0.28\linewidth}
        \centering
        \includegraphics[width=\textwidth, trim={1.5cm 0 1.7cm 0},clip]{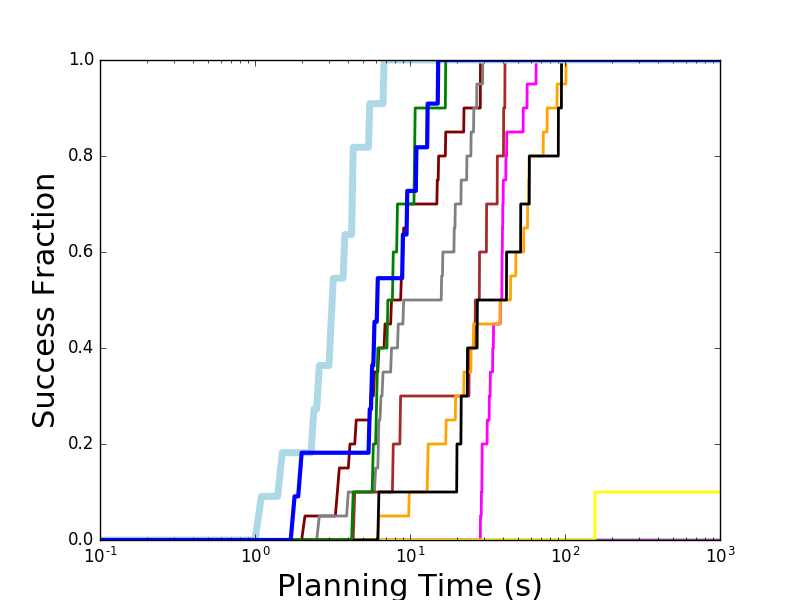}
    \end{subfigure}
    \hfill
    \begin{subfigure}[b]{0.28\linewidth}
        \centering
        \includegraphics[width=\textwidth, trim={1.5cm 0 1.7cm 0},clip]{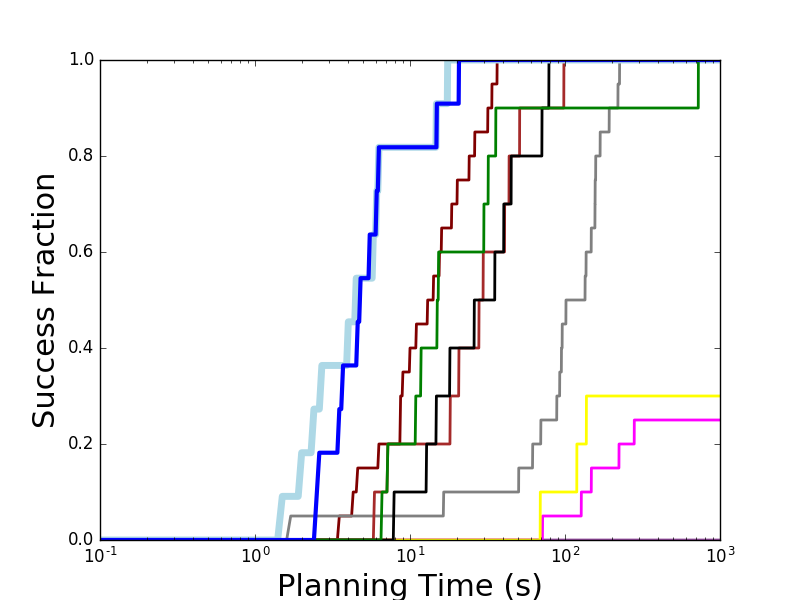}
    \end{subfigure}
    \hfill
    \begin{subfigure}[b]{0.1\linewidth}
        \centering
        \includegraphics[width=\textwidth, trim={16cm 6cm 0cm 0},clip]{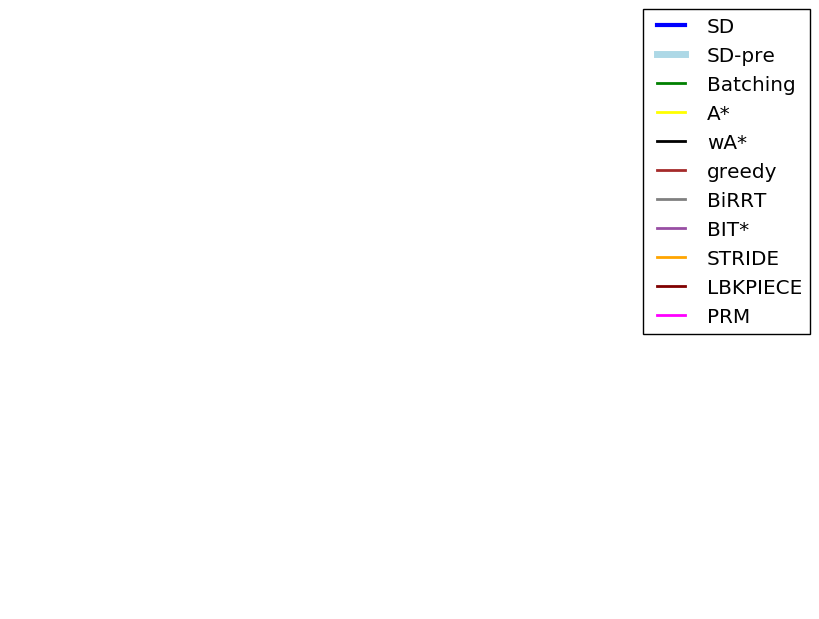}
    \end{subfigure}
    \hfill
    
    \begin{subfigure}[b]{0.03\linewidth}
        \centering
        \includegraphics[width=\textwidth, trim={0cm 2cm 18.8cm 2cm},clip]{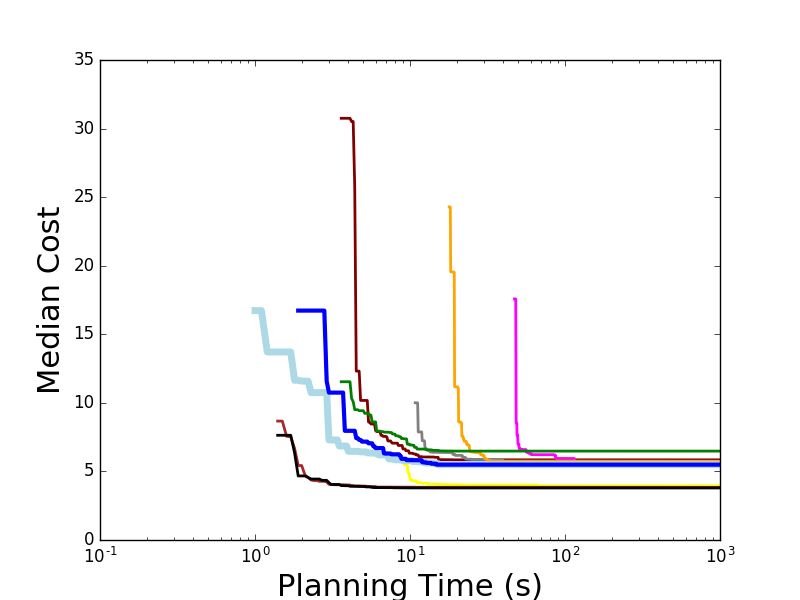}
    \end{subfigure}
    \begin{subfigure}[b]{0.28\linewidth}
        \centering
        \includegraphics[width=\textwidth, trim={1.5cm 0 1.7cm 0cm},clip]{img/Table_medians.png}
    \end{subfigure}
    \hfill
    \begin{subfigure}[b]{0.28\linewidth}
        \centering
        \includegraphics[width=\textwidth, trim={1.5cm 0 1.7cm 0},clip]{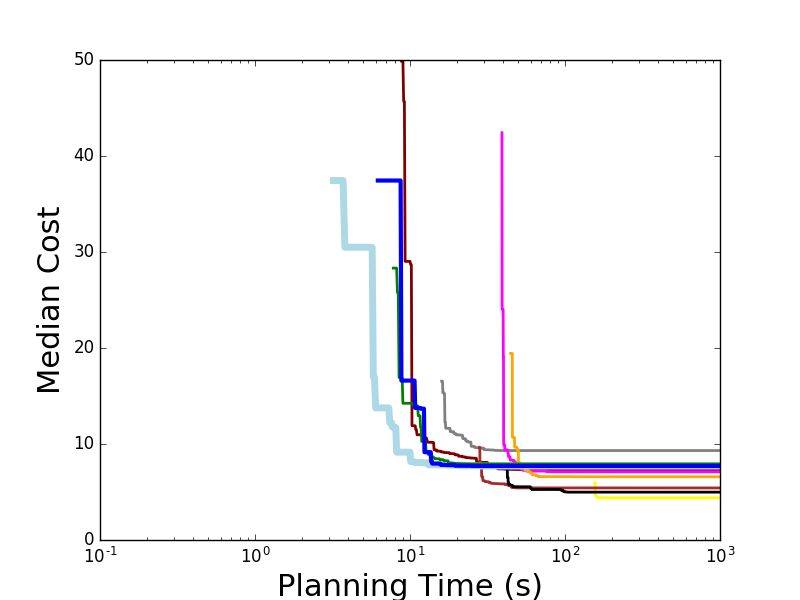}
    \end{subfigure}
    \hfill
    \begin{subfigure}[b]{0.28\linewidth}
        \centering
        \includegraphics[width=\textwidth, trim={1.6cm 0 1.7cm 0},clip]{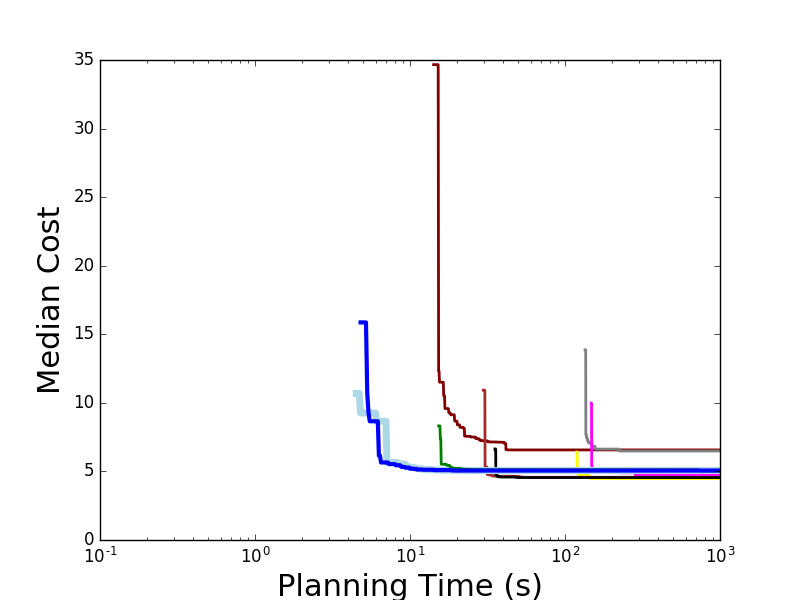}
    \end{subfigure}
    \hfill
    \begin{subfigure}[b]{0.1\linewidth}
        \centering
        \includegraphics[width=\textwidth, trim={16cm 6cm 0cm 0},clip]{img/Table_legend.png}
    \end{subfigure}
    \hfill

    \caption{Planning Comparisons. Top: The robot at the goal with a blue line tracing the end effector path from the start. 
      Middle: Fraction of trials that have achieved any solution, as a function of time.
      Bottom: Median length of path found.
    }

    \label{fig:Experiments}
\end{figure*}

Experiments were performed comparing Selective Densification with many methods of searching the same graph, and also with \bsrevC{numerous algorithms from OMPL} \cite{ompl} on three scenarios for a 7-DOF robotics arm.
\bsrevC{Each strategy was tested 20 times in each scenario with the same start and goal configurations using different random seeds creating different graphs.}
\bsrevC{We considered the cost of a path to be its length in C-space.}

We implemented Selective Densification in C++.
Robot-environment collision checking used GPU-Voxels \cite{Hermann2014} with a voxel grid of 200x200x200 and a voxel size of 2cm$^3$.
Collision checks were performed by intersecting the voxel grid representing the robot with the voxel grid representing the environment obstacles.
A NVidia 1080Ti GPU and an Intel i7-7700K CPU was used for all experiments, and a single collision check took approximately 1ms.

\bsrevC{Each} Layered Graph was constructed explicitly, using $Q$ as a 7D halton sequence \bsrevC{with a pseudo-random offset}.
Layer $i$ contained $n_i = 2^i$ nodes, with connection radius $r_i$ set so the expected number of edges per node was 30, consistent with previous halton roadmap experiments \cite{LazySP}.
\bsrevC{We used 18 layers as empirically this fit comfortably in memory and was able to solve our scenarios}.
Each graph was therefore constructed with 524287 nodes and $\sim$7 million edges, \bsrevC{taking $\sim$2 minutes}.
\bsrevC{Collision checking was performed by discretely checking states along an edge at most 0.02 radians apart.}




We examine Selective Densification with and without precomputation, where the swept volume of relevant edges has been precomputed (SD-pre).
Note that unlike the more common reuse of PRMs by storing edge validity for particular obstacles, this method is agnostic to changes in the environment, and only requires that the robot remain the same.
\bsrevC{
During online planning a collision check can be performed quickly via a lookup of the swept volume followed by an intersection with the environment.}
For the results shown we precomputed all edges that were checked, except from $\node_\start$ and $\node_\goal$, as these were not added until query time.
Checking a precomputed edge took $\sim$1ms, while without precomputation checking took between $\sim$1ms (the first configuration checked was invalid) and $\sim$150ms.
Note that it is infeasible to store the swept volumes for all 7 million $\edges \in \graph$.
\bsrevC{In these experiments we stored all edges checked by all previous planning runs} thus this provides an optimistic assessment of the practical improvement that precomputed swept volumes can offer.

For each scenario we compared Selective Densification to many common methods of graph search used in robotics.
\bsrevC{In all cases we use lazy collision checking, as direct collision checking takes excessively long}
We attempted search using $\astar$ with \bsrevC{the admissible heuristic of C-space distance from Equation \eqref{eq:ah}}, as well as the inflated, and greedy heuristic.
We compare against Iterative Deepening (ID), or batching, where a search continues on a single layer of $\graph$ until that layer is shown to have no solution.
We also compare against \bsrevC{bidirectional RRT-connect, PRM, SPARS, STRIDE, BIT*, and LBKPIECE from OMPL, although all SPARS attempts exceeded our 5 minute limit so results are not presented}.
These approaches have many variants able to improve solutions over time (\bsrevC{e.g.} ANA*, RRT*, etc.), but we present results using shortcut smoothing as we found it vastly outperformed other methods on these problems.
Since shortcut smoothing no longer constrains the path to $\graph$, methods can outperform $\astar$.
\bsrevC{Results are shown in \figref{fig:Experiments}. }
\bsrevC{The deepest layer on which a collision check was performed using SD was 8 (Table), 9 (Bookshelf), and 15 (Slot).}

\bsrevC{
  We find our proposed method (SD and SD-pre) finds a feasible solution to our scenarios in a time comparable to (Table) or faster than (Slot) all other algorithms tested.
  The path length of the initial solution found by SD (and other algorithms) was substantially suboptimal, however a few seconds of shortcut smoothing dramatically reduced the path length in all trials for all methods.
  This suggests that in practice although SD finds paths far better than the bound provided in \thmref{thm:boundedsubopt}, the solution found is far from optimal and will be sensitive to the quality of the smoother.
}

\bsrevC{In \figref{fig:Experiments} we set $\sdw$ (which governs the inflation of each layer of SD) to 1.0}, causing a near-greedy search on deeper layers.
We performed a sweep over $\sdw$ for a single $\graph$, with results for the Table Scenario shown in \figref{fig:sweep}.
As expected, we found that smaller $\sdw$ result in slower search with shorter paths after the graph search.

We find that the largest component of planning time is the collision checking of edges, as expected.
However, we find repeated $\astar$ search time is significant. 
As this LazySP approach repeatedly searches similar graphs, we attempted to reuse information from previous $\astar$ iterations using generalized LPA$^*$ \cite{GLPA} but found this approach slower.

Overall we observed precomputing the swept volume of edges significantly reduces collision checking time of valid edges, however only a modest \bsrevA{decrease}
is seen when checking invalid edges, as we observe in our Scenarios that edges are invalidated typically after checking only a few configurations.
Overall we observed precomputation yields a modest improvement in overall performance.

  Finally, we compare a baseline bidirectional LazySP search to our proposal of balancing the time searching each direction (\algoref{alg:bilazysp}).
In the Table and Bookshelf Scenarios both approaches perform similarly.
In the Slot Scenario the forward search expands far more nodes and edges than the reverse search.
By balancing search time our proposed approach performs fewer forward search iterations before finding a solution using the reverse search (\tabref{tab:bi}).

\begin{figure}
    \centering
    \begin{subfigure}[b]{0.75\linewidth}
        \centering
        \includegraphics[width=\textwidth]{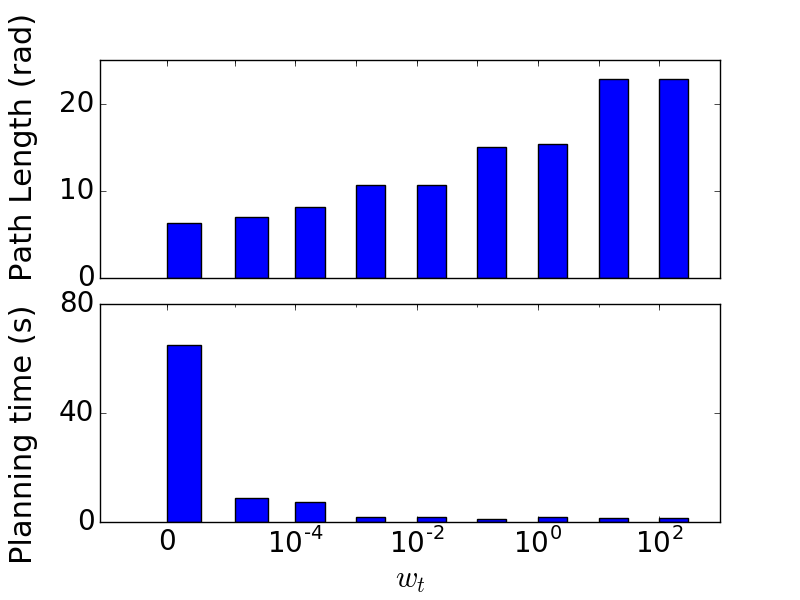}
    \end{subfigure}
    \caption{Comparing Path Length (before smoothing) and planning time on the Table Scenario using Selective Densification with varying $\sdw$.
    }
    \label{fig:sweep}
\end{figure}

\begin{table}
\begin{tabular}{lllll}
 & Total & Edge Check & Forward & Reverse \\
Proposed & 6.5 $\pm$ 5.7 & 3.1 $\pm$ 2.1 & 1.8 $\pm$ 2.0 & 1.4 $\pm$ 1.5\\
Baseline & 73.4 $\pm$ 79.3 & 9.3 $\pm$ 4.4 & 63.0 $\pm$ 75.1 & 0.9 $\pm$ 0.8\\
\end{tabular}
\caption{Planning times in seconds for bidirectional search comparing the Proposed (\algoref{alg:bilazysp}) and Baseline (alternating each iteration) using SD on the Slot Scenario
\label{tab:bi}}
\end{table}


\section{Conclusions and Future Work}

We \bsrevC{proposed and} implemented Selective Densification, a motion planning method that searches a graph composed of layers of different densities of nodes.
SD prioritizes search that is close to the goal and on sparse layers through the planning cost-to-go heuristic $\sdh$.
We presented proofs of path quality and limited search depth and performed planning experiments for a robotic arm demonstrating a speed up when planning in environments requiring dense graphs.


Unexplored in this work is the best method to set $\sdw$, the weighting parameter for $\sdh$.
We showed a larger $\sdw$ tended to lower planning times but increase execution cost of the path found by the graph search.
However, post-processing via shortcut smoothing drastically reduced the path cost suggesting only a marginal benefit might be gained from setting a low $\sdw$.



\bibliographystyle{IEEEtran}
\bibliography{references}

\end{document}